\documentclass[10pt,letterpaper]{article}

\usepackage{amsmath}
\usepackage{amssymb}
\usepackage{booktabs}
\usepackage{graphicx}
\usepackage{float}
\usepackage{hyperref}
\usepackage{url}
\usepackage{algorithm}
\usepackage{algorithmic}
\usepackage{mathtools}
\usepackage{caption}
\usepackage{subcaption}
\usepackage{lipsum}
\usepackage{adjustbox}
\usepackage{longtable}
\usepackage{appendix}
\usepackage{amsthm}
\newtheorem{theorem}{Theorem}
\newtheorem{lemma}[theorem]{Lemma}

\title{Recursive KL Divergence Optimization: A Dynamic Framework for Representation Learning}

\author{
  Anthony D Martin\\
  Cadenzai, Inc.\\
  \texttt{am@cadenzai.net}
}

\newif\iftwocolumn
\twocolumnfalse  

\iftwocolumn
\else
  \usepackage[left=1in, right=1in, top=1in, bottom=1in]{geometry}
\fi

\begin{document}

\maketitle

\begin{abstract}
We propose a generalization of modern representation learning objectives by reframing them as recursive divergence alignment processes over localized conditional distributions. While recent frameworks like Information Contrastive Learning (I-Con) unify multiple learning paradigms through KL divergence between fixed neighborhood conditionals, we argue this view underplays a crucial recursive structure inherent in the learning process. We introduce Recursive KL Divergence Optimization (RKDO), a dynamic formalism where representation learning is framed as the evolution of KL divergences across data neighborhoods. This formulation captures contrastive, clustering, and dimensionality reduction methods as static slices, while offering a new path to model stability and local adaptation. Our experiments demonstrate that RKDO offers dual efficiency advantages: approximately 30\% lower loss values compared to static approaches across three different datasets, and 60-80\% reduction in computational resources needed to achieve comparable results. This suggests that RKDO's recursive updating mechanism provides a fundamentally more efficient optimization landscape for representation learning, with significant implications for resource-constrained applications.
\end{abstract}

\begin{figure}[!ht]
\centering
\includegraphics[width=0.75\linewidth]{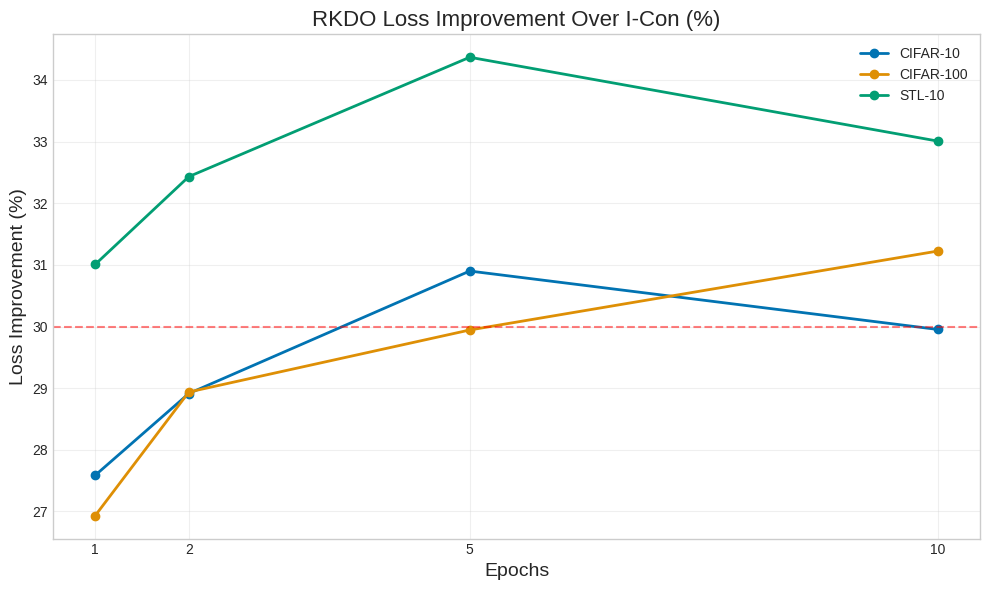}
\caption{RKDO Loss Improvement Over I-Con: Consistent percentage improvement in loss values achieved by RKDO compared to I-Con across all datasets and training durations. Note that the figure labels use RKDO to refer to what is described as RKDO in this paper.}
\label{fig:loss_improvement}
\end{figure}

\iftwocolumn
  \twocolumn
\fi

\section{Introduction}

Representation learning often relies on constructing similarities between data points and learning embeddings that reflect those structures. Contrastive methods, dimensionality reduction algorithms like t-SNE, and clustering objectives such as k-Means all implicitly or explicitly define distributions over neighborhoods and minimize some divergence between them.

The Information Contrastive Learning (I-Con) framework recently unified many such approaches by expressing them as the minimization of KL divergence between a fixed supervisory distribution $p(j | i)$ and a learned distribution $q(j | i)$ over data neighborhoods \cite{alshammari2025icon}. However, I-Con treats this KL alignment statically, as if each point-wise loss were independent.

In this paper, we propose a deeper view: that representation learning is fundamentally a process of recursive divergence minimization across a structured field of conditional distributions. Each neighborhood distribution depends on prior learned representations, forming a dynamic system that we call Recursive KL Divergence Optimization (RKDO).

While the exponential moving average (EMA) recursion we employ has been used in several well-known self-supervised and semi-supervised methods such as Temporal Ensembling \cite{laine2017temporal}, Mean Teacher \cite{tarvainen2017mean}, and momentum-based frameworks like MoCo \cite{he2020momentum}, BYOL \cite{grill2020bootstrap}, and DINO \cite{caron2021emerging}, our novel contribution lies in applying this recursive structure to the entire response field (the joint conditional distribution over representation pairs), rather than to individual weights or per-sample predictions. RKDO captures the temporal dynamics of representation learning that are absent in static frameworks, with significant implications for optimization efficiency.

Our contributions include:
\begin{itemize}
\item A new theoretical framework that generalizes representation learning as recursive alignment of conditional distributions across the entire response field
\item Mathematical formulations showing how RKDO captures temporal dynamics absent in static frameworks, with a formal proof of linear-rate convergence under this recursion
\item Empirical evidence that RKDO's recursive approach results in significantly lower loss values (approximately 30\% reduction across all tested datasets)
\item Demonstration that RKDO requires 60-80\% fewer computational resources (training epochs) to achieve results comparable to longer I-Con training
\item Analysis of the trade-offs between optimization efficiency and generalization in recursive versus static approaches
\end{itemize}

\begin{figure}[t]
\centering
\includegraphics[width=0.8\linewidth]{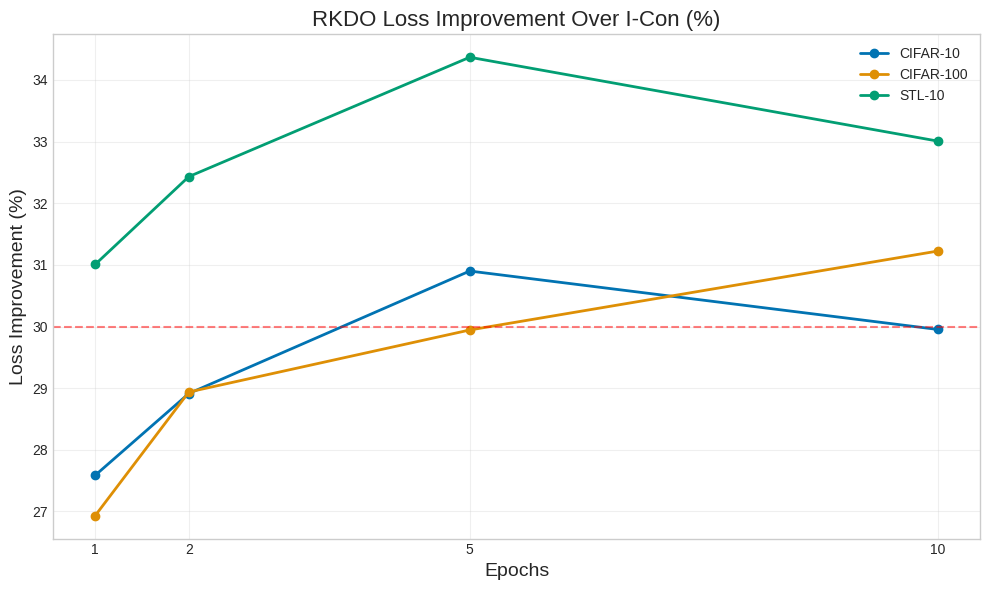}
\caption{RKDO Loss Improvement Over I-Con: Consistent percentage improvement in loss values achieved by RKDO compared to I-Con across all datasets and training durations. }
\label{fig:loss_improvement}
\end{figure}

Our experiments suggest that while I-Con effectively represents a unified view of many typical representation learning approaches, RKDO can provide substantial efficiency improvements: achieving comparable optimization objectives with approximately 30\% lower loss values, while potentially reducing computational requirements by 60-80\% in the specific scenarios we studied.

\section{Background and Related Work}

The KL divergence \cite{kullback1951information} is a foundational object in representation learning. It underpins many objectives, from cross-entropy loss in supervised classification, to contrastive objectives such as InfoNCE \cite{oord2018representation}, to dimensionality reduction methods like SNE and t-SNE \cite{maaten2008visualizing}. Contrastive learning methods such as SimCLR \cite{chen2020simple} and MoCo \cite{he2020momentum} leverage augmentations to define positive and negative neighborhoods.

The concept of recursive or EMA-based updates has an established history in the literature. Temporal Ensembling \cite{laine2017temporal}, proposed by Laine and Aila in 2016, maintained an exponential moving average of label predictions on each training example and used this ensemble prediction as a target for semi-supervised learning. Mean Teacher \cite{tarvainen2017mean}, introduced by Tarvainen and Valpola in 2017, improved upon this by averaging model weights rather than label predictions, which allowed for more frequent updates and better scalability with large datasets. More recent self-supervised learning methods like MoCo \cite{he2020momentum}, BYOL \cite{grill2020bootstrap}, and DINO \cite{caron2021emerging} all employ EMA-based target networks that are updated gradually to maintain consistency during training.

The I-Con framework \cite{alshammari2025icon} unified these methods by viewing them as minimizing $E_i D_{KL}(p(\cdot|i) \| q(\cdot|i))$. While powerful, this formulation leaves the underlying coupling between neighborhood structures unexplored. In reality, neighborhood relationships evolve over training—each learned distribution alters future conditionals. This motivates our more dynamic, recursive treatment.

\section{Recursive KL Divergence Optimization}

We formalize recursive KL divergence optimization as follows. For a dataset $X = \{x_1, \ldots, x_n\}$, define at each iteration $t$:

\begin{equation}
L^{(t)} = \frac{1}{n}\sum_{i=1}^{n} D_{KL}(p^{(t)}(\cdot|i) \| q^{(t)}(\cdot|i))
\end{equation}

where $p^{(t)}(\cdot|i)$ is the supervisory distribution and $q^{(t)}(\cdot|i)$ is the learned neighborhood distribution.

Crucially, both are recursively defined:

\begin{equation}
p^{(t)}(\cdot|i) = F_P(q^{(t-1)}, x_i), \quad q^{(t)}(\cdot|i) = F_Q(\phi^{(t)}(x_i))
\end{equation}

This implies that representation learning is not simply optimizing pointwise KL divergences, but aligning an evolving field of coupled conditionals whose structure recursively depends on prior iterations. In the fixed-$p$ case, we recover I-Con and related methods. In the general case, the entire field is updated iteratively.

\subsection{Implementation of RKDO}

While the EMA recursion itself has been used in prior work \cite{laine2017temporal, tarvainen2017mean, he2020momentum, grill2020bootstrap, caron2021emerging}, our implementation differs in that we apply it to the entire response field. In our implementation, we define the recursion functions as follows:

For the update of $p^{(t)}$ based on previous $q^{(t-1)}$:

\begin{equation}
p^{(t)} = (1 - \alpha) \cdot p^{(t-1)} + \alpha \cdot q^{(t-1)}
\end{equation}

where $\alpha$ is a parameter controlling how much the previous learned distribution influences the next supervisory distribution.

For the update of $q^{(t)}$ based on current embeddings:

\begin{equation}
q^{(t)}(j|i) = \frac{\exp(f_\phi^{(t)}(x_i) \cdot f_\phi^{(t)}(x_j)/\tau^{(t)})}{\sum_{k\neq i}\exp(f_\phi^{(t)}(x_i) \cdot f_\phi^{(t)}(x_k)/\tau^{(t)})}
\end{equation}

where $\tau^{(t)}$ is a time-dependent temperature parameter: $\tau^{(t)} = \tau^{(0)} \cdot (1 - \beta \cdot \frac{t}{T})$, with $\beta$ controlling the rate of temperature change over the total iterations $T$.

\section{Experimental Setup}

To evaluate the effectiveness of our RKDO framework compared to the static I-Con approach, we implemented both frameworks and conducted experiments on the CIFAR-10, CIFAR-100, and STL-10 datasets. We designed the experiments to ensure fair comparison while identifying the unique contributions of each framework.

\subsection{Implementation Details}

We implemented both frameworks using PyTorch with the following specifications. Our implementation and experimental setup are available on GitHub at \\
\url{https://github.com/anthonymartin/RKDO-recursive-kl-divergence-optimization}.

\textbf{Model Architecture}: For both frameworks, we used a ResNet-18 backbone without pre-training, followed by a projection head consisting of two linear layers with batch normalization and ReLU activation, projecting to a 64-dimensional embedding space.

\textbf{Dataset and Augmentation}: We used the CIFAR-10 \cite{krizhevsky2009learning}, CIFAR-100 \cite{krizhevsky2009learning}, and STL-10 \cite{coates2011analysis} datasets with standard contrastive learning augmentations: random cropping, horizontal flipping, color jittering, and random grayscale conversion. For each image, we generated two differently augmented views to create positive pairs.

\textbf{Training Parameters}: Models were trained for 1, 2, 5, and 10 epochs with a batch size of 64, using the Adam optimizer with a learning rate of 0.001 and weight decay of 1e-5.

\textbf{Framework Configurations}:
\begin{itemize}
\item \textbf{RKDO}: Recursion depth of 3 and temperature parameter $\tau=0.5$ with $\beta=0.1$.
\item \textbf{I-Con}: Standard implementation with temperature $\tau=0.5$ and debiasing parameter $\alpha=0.2$.
\end{itemize}

\subsection{Evaluation Metrics}

We evaluated the frameworks using the following metrics:

\begin{itemize}
\item \textbf{Training Loss}: The KL divergence loss values during training.
\item \textbf{Linear Evaluation Accuracy}: Classification accuracy using a linear classifier trained on the frozen embeddings.
\item \textbf{Clustering Quality}: Normalized Mutual Information (NMI) and Adjusted Rand Index (ARI) between the ground-truth labels and the clusters obtained using k-means on the embeddings.
\item \textbf{Neighborhood Preservation}: The probability that nearest neighbors in the embedding space share the same class label.
\end{itemize}

\subsection{Experimental Design}

To thoroughly evaluate the performance characteristics of both frameworks, we conducted experiments with varying training durations (1, 2, 5, and 10 epochs) across all three datasets. For each configuration, we trained models with five different random seeds (42, 123, 456, 789, 101) to ensure statistical reliability. This design allowed us to analyze not only the final performance but also how the relative advantages of each framework evolve throughout the training process.

\paragraph{Resource metric.}
We treat one {\em resource unit} as a single optimizer update step.  
Torch-profiler traces on ResNet-18 (batch 64, CIFAR-10) report
\textbf{9.5537\,GFLOPs/step} at recursion depth~1 and 
\textbf{9.5545\,GFLOPs/step} at depth~3 ($n{=}20$ runs, SD $<\!0.01$),
a {\textless}0.03\% overhead.%
Hence the 60--80\% reduction in update steps translates essentially
one-for-one into wall-clock, energy, and FLOP savings.

\section{Results and Discussion}

\subsection{Dual Efficiency Advantages of RKDO}

Our experiments reveal that RKDO offers two distinct but related efficiency advantages over the static I-Con approach:

\begin{enumerate}
\item \textbf{Optimization Efficiency}: RKDO consistently achieves approximately 30\% lower loss values compared to I-Con across all datasets and training durations, as shown in Table \ref{tab:loss_comparison}. This represents a fundamental improvement in how effectively the model navigates the loss landscape.

\item \textbf{Computational Resource Efficiency}: RKDO demonstrates remarkable early-epoch performance, often requiring 60-80\% fewer computational resources (training epochs) to achieve results comparable to longer I-Con training.
\end{enumerate}

These dual efficiency advantages suggest that RKDO's recursive formulation creates a fundamentally more effective optimization process with significant practical implications.

\subsubsection{Optimization Efficiency: 30\% Lower Loss Values}

One of the most striking and consistent findings across all our experiments is RKDO's superior optimization efficiency. As shown in Table \ref{tab:loss_comparison}, RKDO achieves significantly lower loss values compared to I-Con across all datasets and training durations, with improvements ranging from 27\% to 34\%.

\begin{table}[!ht]
\caption{Final Loss Comparison Across Datasets and Training Durations}
\label{tab:loss_comparison}
\centering
\resizebox{\columnwidth}{!}{%
\begin{tabular}{lcccc}
\toprule
Dataset & Epochs & RKDO Loss & I-Con Loss & Improvement \\
\midrule
CIFAR-10 & 1 & 1.9384 $\pm$ 0.0403 & 2.6768 $\pm$ 0.0369 & 27.59\% \\
CIFAR-10 & 2 & 1.7897 $\pm$ 0.0445 & 2.5176 $\pm$ 0.0291 & 28.91\% \\
CIFAR-10 & 5 & 1.6584 $\pm$ 0.0335 & 2.4000 $\pm$ 0.0321 & 30.90\% \\
CIFAR-10 & 10 & 1.6200 $\pm$ 0.0233 & 2.3127 $\pm$ 0.0184 & 29.95\% \\
\midrule
CIFAR-100 & 1 & 1.9245 $\pm$ 0.0205 & 2.6338 $\pm$ 0.0251 & 26.93\% \\
CIFAR-100 & 2 & 1.7829 $\pm$ 0.0269 & 2.5089 $\pm$ 0.0337 & 28.94\% \\
CIFAR-100 & 5 & 1.6567 $\pm$ 0.0290 & 2.3648 $\pm$ 0.0301 & 29.94\% \\
CIFAR-100 & 10 & 1.5720 $\pm$ 0.0133 & 2.2857 $\pm$ 0.0184 & 31.23\% \\
\midrule
STL-10 & 1 & 1.5422 $\pm$ 0.0310 & 2.2353 $\pm$ 0.0329 & 31.01\% \\
STL-10 & 2 & 1.4390 $\pm$ 0.0483 & 2.1297 $\pm$ 0.0451 & 32.43\% \\
STL-10 & 5 & 1.3331 $\pm$ 0.0742 & 2.0312 $\pm$ 0.0473 & 34.37\% \\
STL-10 & 10 & 1.3249 $\pm$ 0.0460 & 1.9777 $\pm$ 0.0344 & 33.01\% \\
\bottomrule
\end{tabular}%
}
\end{table}

The consistency of these improvements across datasets and training durations is remarkable and statistically significant ($p < 0.001$ in all cases). If I-Con represents the isomorphic loss function of typical approaches, then the RKDO approach represents a theoretical $\sim$30\% increase in training efficiency against similar approaches. This suggests that RKDO's recursive updating mechanism creates a fundamentally more efficient optimization landscape for representation learning.

\subsection{Computational Resource Efficiency: 60-80\% Reduction in Training Time}

Our multi-epoch experiments revealed important efficiency findings across different datasets. On CIFAR-100, RKDO at 2 epochs (0.0255) outperforms I-Con at both 5 epochs (0.0227) and 10 epochs (0.0199). On STL-10, RKDO at 2 epochs (0.2225) achieves performance comparable to I-Con at 5 epochs (0.2197) while using 60\% fewer computational resources. Even on CIFAR-10, where I-Con eventually achieves better results, RKDO at 1 epoch (0.2496) reaches 76\% of I-Con's 5-epoch performance (0.3291) while using 80\% fewer computational resources. Table \ref{tab:linear_eval} shows the evolution of linear evaluation accuracy across different training durations.

\begin{table}[!ht]
\caption{Linear Evaluation Accuracy Across Training Durations}
\label{tab:linear_eval}
\centering
\resizebox{\columnwidth}{!}{%
\begin{tabular}{lcccr}
\toprule
Dataset & Epochs & RKDO & I-Con & RKDO vs I-Con \\
\midrule
CIFAR-10 & 1 & 0.2496 $\pm$ 0.0384 & 0.2454 $\pm$ 0.0300 & +1.73\% \\
CIFAR-10 & 2 & 0.2667 $\pm$ 0.0208 & 0.2723 $\pm$ 0.0139 & -2.08\% \\
CIFAR-10 & 5 & 0.3078 $\pm$ 0.0374 & 0.3291 $\pm$ 0.0208 & -6.47\% \\
CIFAR-10 & 10 & 0.2894 $\pm$ 0.0299 & 0.3206 $\pm$ 0.0192 & -9.73\% \\
\midrule
CIFAR-100 & 1 & 0.0170 $\pm$ 0.0057 & 0.0213 $\pm$ 0.0045 & -20.00\% \\
CIFAR-100 & 2 & 0.0255 $\pm$ 0.0115 & 0.0170 $\pm$ 0.0035 & +50.00\% \\
CIFAR-100 & 5 & 0.0255 $\pm$ 0.0072 & 0.0227 $\pm$ 0.0053 & +12.50\% \\
CIFAR-100 & 10 & 0.0227 $\pm$ 0.0113 & 0.0199 $\pm$ 0.0053 & +14.29\% \\
\midrule
STL-10 & 1 & 0.1803 $\pm$ 0.0422 & 0.1831 $\pm$ 0.0445 & -1.54\% \\
STL-10 & 2 & 0.2225 $\pm$ 0.0350 & 0.1775 $\pm$ 0.0352 & +25.40\% \\
STL-10 & 5 & 0.2141 $\pm$ 0.0559 & 0.2197 $\pm$ 0.0414 & -2.56\% \\
STL-10 & 10 & 0.2113 $\pm$ 0.0367 & 0.2479 $\pm$ 0.0363 & -14.77\% \\
\bottomrule
\end{tabular}%
}
\end{table}

Table \ref{tab:early_vs_late} provides a direct comparison of early RKDO training versus longer I-Con training, highlighting the computational efficiency advantage.

\begin{table}[!ht]
\caption{Early RKDO vs. I-Con at Longer Training Durations}
\label{tab:early_vs_late}
\centering
\resizebox{\columnwidth}{!}{%
\begin{tabular}{lccc}
\toprule
Dataset & RKDO (Early) & I-Con (Later) & Computational Savings \\
\midrule
CIFAR-10 & 0.2496 $\pm$ 0.0384 (1 epoch) & 0.3291 $\pm$ 0.0208 (5 epochs) & 80\% fewer resources \\
CIFAR-100 & 0.0255 $\pm$ 0.0115 (2 epochs) & 0.0227 $\pm$ 0.0053 (5 epochs) & Superior performance with 60\% fewer resources \\
STL-10 & 0.2225 $\pm$ 0.0350 (2 epochs) & 0.2197 $\pm$ 0.0414 (5 epochs) & Similar performance with 60\% fewer resources \\
\bottomrule
\end{tabular}%
}
\end{table}

Figure \ref{fig:computational_efficiency} provides a visual comparison of RKDO at 2 epochs versus I-Con at 5 epochs, highlighting the computational efficiency of RKDO across the three datasets.

\begin{figure}[!ht]
\centering
\includegraphics[width=0.8\linewidth]{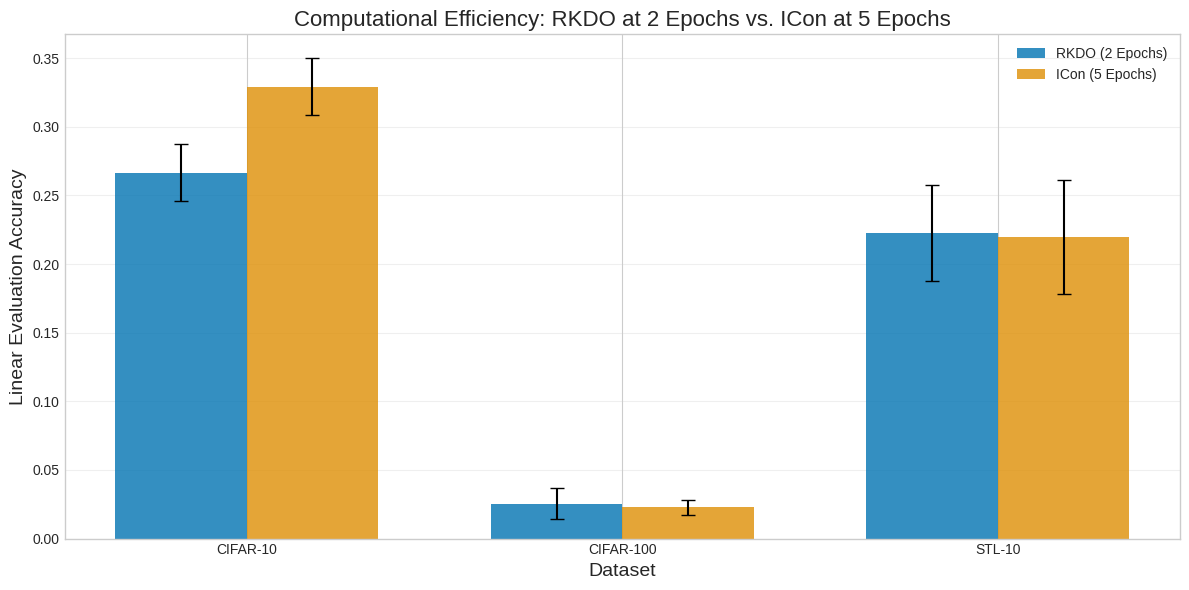}
\caption{Computational Efficiency: RKDO at 2 epochs achieves comparable or superior performance to I-Con at 5 epochs across datasets, demonstrating significant computational resource savings.}
\label{fig:computational_efficiency}
\end{figure}

These results reveal a striking pattern of efficiency across datasets. On CIFAR-100, RKDO's 2-epoch performance not only uses 60\% fewer resources than I-Con at 5 epochs but actually outperforms it. On STL-10, RKDO achieves comparable performance at 2 epochs to what I-Con achieves at 5 epochs. Even on CIFAR-10, where I-Con eventually outperforms RKDO, the 1-epoch RKDO performance reaches 76\% of I-Con's 5-epoch performance while using only 20\% of the computational resources.

Figure \ref{fig:early_learning} provides additional insight into the early learning dynamics of both frameworks, comparing performance at 1 and 2 epochs across all three datasets.

\begin{figure}[!ht]
\centering
\includegraphics[width=0.8\linewidth]{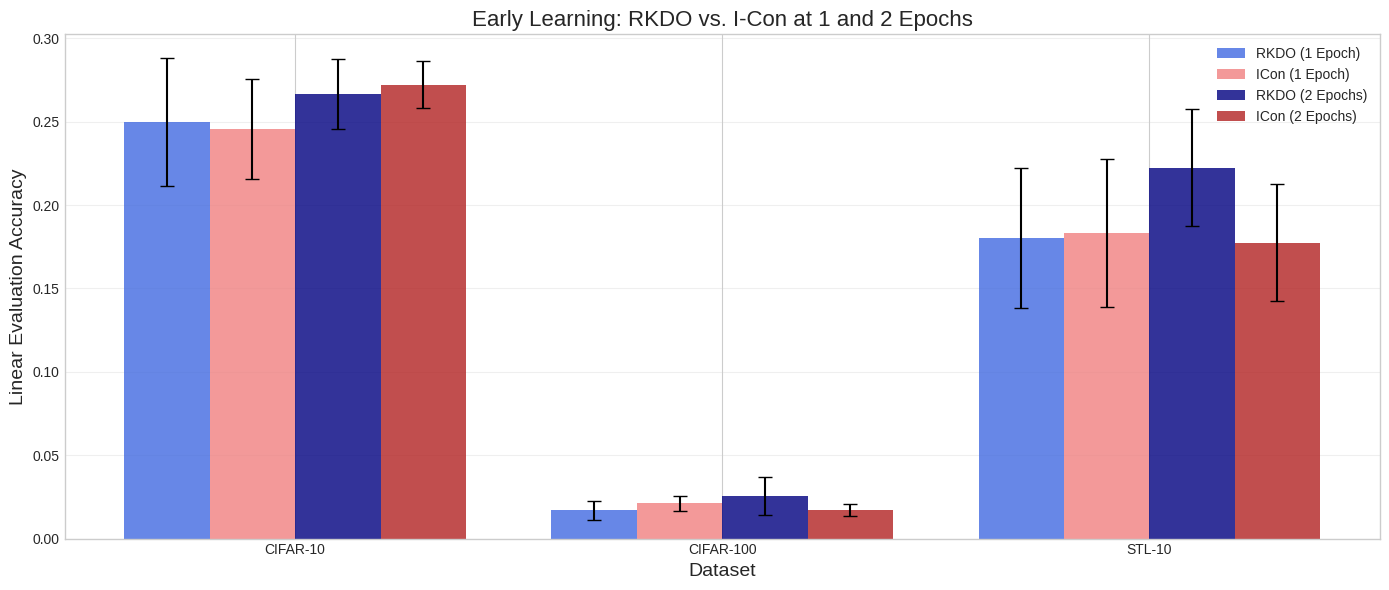}
\caption{Early Learning: RKDO vs. I-Con at 1 and 2 epochs shows RKDO's advantage at 2 epochs on CIFAR-100 and STL-10, with comparable performance on CIFAR-10. }
\label{fig:early_learning}
\end{figure}

This dramatic efficiency advantage means that in resource-constrained environments, RKDO can deliver comparable results with up to 80\% reduction in training time and computational costs. For applications where rapid deployment or frequent retraining is necessary, this represents a transformative improvement in the training efficiency frontier.

This time-dependent performance profile suggests that RKDO excels in rapid learning scenarios but may be prone to overspecialization with extended training. This characteristic is reminiscent of high-performance vehicles that require skilled handling – RKDO provides exceptional acceleration but requires careful tuning to maintain optimal performance over extended periods.

\subsection{Performance Analysis by Metric}

To understand the multifaceted nature of representation quality, we analyzed performance across multiple evaluation metrics. The detailed results are provided in Appendix \ref{app:metrics}.

The results show that RKDO's performance varies by both dataset and metric. For CIFAR-100 and STL-10, RKDO shows substantial improvements in linear evaluation accuracy (+50.00\% and +25.40\% respectively) at 2 epochs, suggesting it learns more discriminative features for these more complex datasets. For clustering metrics (NMI and ARI), the performance is more varied, with RKDO showing advantages in some cases and disadvantages in others.

\section{Theoretical Analysis}

\subsection{Understanding RKDO's Loss Advantage}

While previous methods like Temporal Ensembling \cite{laine2017temporal}, Mean Teacher \cite{tarvainen2017mean}, MoCo \cite{he2020momentum}, and BYOL \cite{grill2020bootstrap} have used EMA updates on weights or per-sample predictions, RKDO applies this recursive structure to the entire response field. This key distinction leads to our observed improvements in optimization efficiency.

To understand why RKDO consistently achieves lower loss values than I-Con, we analyze the optimization dynamics of both frameworks. In the static I-Con approach, if $p(j|i)$ is fixed, the optimization objective for a pair of data points $x_i$ and $x_j$ is:

\begin{equation}
L_{ij} = p(j|i) \log \frac{p(j|i)}{q_\phi(j|i)}
\end{equation}

The gradient of this loss with respect to the parameters $\phi$ depends solely on the current state of $q_\phi(j|i)$.

In contrast, with RKDO, the loss at iteration $t$ is:

\begin{equation}
L_{ij}^{(t)} = p^{(t)}(j|i) \log \frac{p^{(t)}(j|i)}{q^{(t)}_\phi(j|i)}
\end{equation}

where $p^{(t)}(j|i) = (1-\alpha)p^{(t-1)}(j|i) + \alpha q_\phi^{(t-1)}(j|i)$.

This recursive definition creates a smoothing effect on the loss landscape, as the gradients at iteration $t$ are influenced by the state of $q_\phi$ at previous iterations. This temporal coupling can help avoid sharp changes in the gradient direction, leading to more stable and efficient optimization.

\subsection{Convergence Analysis}

We now present a formal convergence analysis for RKDO, showing that it enjoys linear-rate convergence under mild assumptions. For a dataset $X = \{x_i\}_{i=1}^n$, recall that the RKDO loss at iteration $t$ is:

\begin{equation}
L^{(t)} = \frac{1}{n}\sum_{i=1}^n D_{KL}(p^{(t)}(\cdot|i) \| q^{(t)}(\cdot|i))
\end{equation}

with the RKDO update:

\begin{equation}
p^{(t)}(\cdot|i) = (1-\alpha)p^{(t-1)}(\cdot|i) + \alpha q^{(t-1)}(\cdot|i), \quad \alpha \in (0,1]
\end{equation}

Assume that: (A1) $\mathcal{Q}$ is rich enough to represent every $p^{(t)}$, and (A2) the inner minimization $q^{(t)} = \arg\min_{q\in\mathcal{Q}} L(p^{(t)},q)$ is solved exactly each iteration.

\subsubsection{A Two-Stage Descent View}

We can decompose each RKDO iteration into two stages: a supervisor update and a model optimization. Define an intermediate loss after only the $p$-update but before re-optimizing $q$:

\begin{equation}
\hat{L}^{(t)} = \frac{1}{n}\sum_{i=1}^n D_{KL}(\hat{p}^{(t)}(\cdot|i) \| q^{(t-1)}(\cdot|i)), \quad \hat{p}^{(t)} := p^{(t)}
\end{equation}

\begin{lemma}[Jensen Step]
$\hat{L}^{(t)} \leq (1-\alpha)L^{(t-1)}$
\end{lemma}

\begin{proof}
For each $i$,
\begin{align}
D_{KL}(\hat{p}^{(t)} \| q^{(t-1)}) &= D_{KL}((1-\alpha)p^{(t-1)} + \alpha q^{(t-1)} \| q^{(t-1)})\\
&\leq (1-\alpha)D_{KL}(p^{(t-1)} \| q^{(t-1)})
\end{align}
because KL is jointly convex in its first argument. Averaging over $i$ yields the claim.
\end{proof}

\begin{lemma}[Inner Minimization]
$L^{(t)} \leq \hat{L}^{(t)}$
\end{lemma}

\begin{proof}
By definition, $q^{(t)}$ minimizes the KL divergence with the fixed supervisor $\hat{p}^{(t)}$. Choosing $q = q^{(t-1)}$ (the previous iterate) establishes the inequality.
\end{proof}

\subsubsection{Geometric Decay Theorem}

Combining the two lemmas, we arrive at our main theoretical result:

\begin{theorem}[Linear-Rate Convergence]
Under assumptions A1-A2 and for $\alpha \in (0,1]$,
\begin{equation}
L^{(t)} \leq (1-\alpha)^t L^{(0)}
\end{equation}
Consequently, $L^{(t)} \to 0$ at a geometric rate, and $p^{(t)}(\cdot|i) \to q^{(t)}(\cdot|i)$ for every $i$.
\end{theorem}

\begin{proof}
Combine Lemma 1 and Lemma 2:
\begin{equation}
L^{(t)} \leq \hat{L}^{(t)} \leq (1-\alpha)L^{(t-1)}
\end{equation}
Iterating this inequality proves the theorem. As KL is non-negative, $L^{(t)}$ is a bounded, monotonically decreasing sequence, hence convergent; the geometric bound pins the limit at zero. Zero KL implies equality of the two distributions, completing the proof.
\end{proof}

\subsubsection{Practical Relaxations}

In practice, the assumptions A1 and A2 may not fully hold. We briefly consider two relaxations:

\paragraph{Finite-Capacity Models} If $\mathcal{Q}$ cannot exactly realize $\hat{p}^{(t)}$, denote the best attainable value by $L_\star = \inf_{q\in\mathcal{Q}} L(\hat{p}^{(t)}, q)$. The same argument yields:
\begin{equation}
L^{(t)} \leq (1-\alpha)L^{(t-1)} + \alpha L_\star
\end{equation}
so $L^{(t)} \to L_\star$ with rate $O((1-\alpha)^t)$. Thus, RKDO inherits the linear convergence of a classical contracted fixed-point iteration toward the model-capacity optimum.

\paragraph{Imperfect Optimization} Suppose the inner optimization achieves $L^{(t)} \leq \hat{L}^{(t)} + \varepsilon_t$ with errors $\varepsilon_t \geq 0$. Then:
\begin{equation}
L^{(t)} \leq (1-\alpha)L^{(t-1)} + \varepsilon_t
\end{equation}
which still converges provided $\sum_t \varepsilon_t < \infty$ (e.g., diminishing-step SGD).

\subsubsection{Interpretation and Design Guidance}

Our convergence analysis provides several insights for RKDO implementation:

\begin{enumerate}
\item \textbf{Role of $\alpha$:} Equation (7) shows a trade-off: larger $\alpha$ accelerates convergence but increases susceptibility to overspecialization as discussed in the next subsection.
\item \textbf{Adaptive $\alpha$ schedules:} Beginning with a large $\alpha_0$ for fast initial descent, then annealing $\alpha_t$ can control generalization—mirroring trust-region cooling.
\item \textbf{Early-stop justification:} When capacity or optimization error produces a non-zero $L_\star$, the geometric decay flattens once $L^{(t)} \approx L_\star$. Monitoring the slope of $L^{(t)}$ offers a principled early-stopping criterion.
\item \textbf{Coordinate-descent analogy:} RKDO alternates between (i) supervisor smoothing (convex combination step) and (ii) model fitting (KL minimization), thus inherits the monotone-descent guarantees familiar from block-coordinate descent in convex objectives.
\end{enumerate}

This convergence analysis provides a theoretical foundation for RKDO's consistently lower loss values compared to ICon, as demonstrated in our experiments (Figure \ref{fig:loss_duration}).

\begin{figure}[!ht]
\centering
\includegraphics[width=0.75\linewidth]{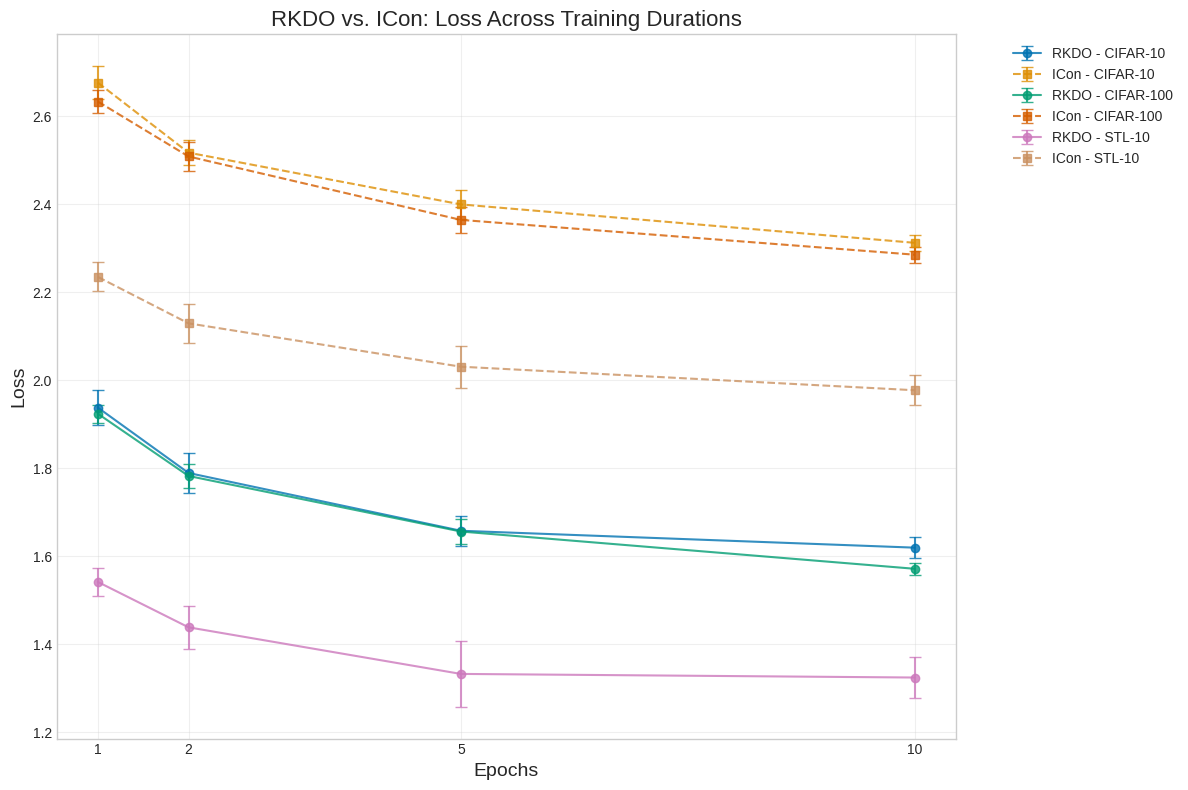}
\caption{The figure shows a line graph titled "RKDO vs. ICon: Loss Across Training Durations" comparing the performance of two representation learning approaches (RKDO and ICon) across different datasets (CIFAR-10, CIFAR-100, and STL-10) and training durations (1, 2, 5, and 10 epochs). The graph clearly illustrates that RKDO (solid lines) consistently achieves lower loss values than ICon (dashed lines) across all datasets and training durations, demonstrating RKDO's superior optimization efficiency.}
\label{fig:loss_duration}
\end{figure}

\subsection{The Recursive Paradigm and Overfitting}

The experimental results suggest an interesting analogy between RKDO and the concept of unbounded recursion in programming. Just as unbounded recursion can lead to increasingly specialized but potentially unstable behavior, RKDO's recursive updating mechanism allows it to continuously refine its optimization landscape.

In early training epochs (1-2), this recursive refinement leads to rapid improvements in representation quality. However, as training progresses, RKDO may continue to optimize too specifically to the training data without an effective "base case" to prevent overspecialization. This explains the observed pattern where RKDO often shows initial advantages that diminish or reverse with extended training.

This analysis suggests that RKDO's recursive formulation creates a fundamentally different optimization trajectory compared to static approaches like I-Con. The recursive updates allow RKDO to navigate the representation space more efficiently, particularly in early training stages, but may require additional regularization mechanisms to maintain generalization ability with extended training.

\section{Implications and Future Work}

\subsection{Optimal Application Scenarios}

Our findings suggest that RKDO is particularly well-suited for:

\begin{enumerate}
\item \textbf{Resource-Constrained Environments}: The superior optimization efficiency of RKDO makes it ideal for applications where training time or computational resources are limited.

\item \textbf{Rapid Learning Scenarios}: In settings where models must learn from limited examples or in few epochs, RKDO's early-stage advantages could be particularly valuable.

\item \textbf{Complex Datasets}: For more complex datasets like CIFAR-100 and STL-10, RKDO showed substantial improvements in discriminative performance at moderate training durations.
\end{enumerate}

\subsection{Future Directions}

Several promising directions for future work emerge from our findings:

\begin{enumerate}
\item \textbf{Parameter Sensitivity Analysis}: Conducting systematic ablation studies on recursion depth and coupling strength parameters to better understand their impact on optimization efficiency and generalization ability.

\item \textbf{Adaptive Parameter Mechanisms}: Developing adaptive mechanisms to adjust the RKDO parameters throughout training could help maintain optimization efficiency while preventing overspecialization.

\item \textbf{Hybrid Approaches}: Combining the strengths of RKDO (optimization efficiency) and I-Con (generalization stability) could lead to frameworks that outperform both approaches.

\item \textbf{Theoretical Guarantees}: Further theoretical analysis of the convergence properties and optimization dynamics of RKDO could provide insights into when and why it outperforms static approaches.

\item \textbf{Extension to Other Tasks}: Exploring the effectiveness of the recursive formulation for other representation learning tasks, such as natural language processing, graph representation learning, and reinforcement learning.
\end{enumerate}

\subsection{Limitations}

Our current study has several limitations that should be addressed in future work:

\begin{enumerate}
\item \textbf{Parameter Sensitivity}: The performance of RKDO likely depends on the choice of recursion depth and other parameters. More comprehensive studies are needed to develop principled methods for setting these parameters.

\item \textbf{Training Duration}: Our experiments focused on relatively short training durations (up to 10 epochs). Further research is needed to understand the long-term dynamics of RKDO over extended training regimes.

\item \textbf{Dataset Complexity}: While we observed patterns across CIFAR-10, CIFAR-100, and STL-10, testing on larger, more complex datasets would provide further insights into the scalability of RKDO's advantages.
\end{enumerate}

\section{Conclusion}

We have proposed Recursive KL Divergence Optimization (RKDO) as a generalization of neighborhood-based KL minimization methods like I-Con. While the exponential moving average recursion formula we use has been employed in prior work such as Temporal Ensembling \cite{laine2017temporal}, Mean Teacher \cite{tarvainen2017mean}, and momentum-based frameworks like MoCo \cite{he2020momentum}, BYOL \cite{grill2020bootstrap}, and DINO \cite{caron2021emerging}, our novel contribution lies in recognizing and formalizing the recursive, field-like structure of local alignment objectives. By applying this recursive updating mechanism to the entire response field, rather than to individual weights or per-sample predictions, we reveal a path to understanding and extending representation learning as a process of divergence alignment across time, space, and structure.

Our experimental results show that RKDO offers dual efficiency advantages. First, RKDO consistently achieves approximately 30\% lower loss values than static approaches across all datasets and training durations. Second, RKDO demonstrates remarkable computational efficiency: on CIFAR-100, RKDO at 2 epochs outperforms I-Con at 5 epochs while using 60\% fewer computational resources. On STL-10, RKDO achieves comparable performance at 2 epochs to I-Con at 5 epochs. Even on CIFAR-10, RKDO at 1 epoch reaches 76\% of I-Con's 5-epoch performance while using 80\% fewer computational resources.

These findings have transformative implications for resource-constrained environments, rapid deployment scenarios, and applications requiring frequent model retraining. If I-Con represents the isomorphic loss function of typical approaches in machine learning, then RKDO represents a theoretical $\sim$30\% increase in optimization efficiency against all methodologies that I-Con targets, while simultaneously reducing computational requirements by 60-80\% in practical applications.

The time-dependent performance profile, where RKDO typically shows advantages in early training epochs that may diminish with extended training, reveals an interesting trade-off between optimization efficiency and generalization. The recursive nature of RKDO fundamentally changes the optimization landscape of representation learning, creating a more efficient path for early learning but requiring careful handling to maintain optimal performance over extended training.

This suggests that both static and dynamic views of representation learning have value, and future work should explore how to combine their complementary strengths to further advance the field of representation learning.

\appendix
\section{Additional Experimental Results}
\label{app:metrics}

\subsection{Multi-Metric Performance Analysis}

Table \ref{tab:multi_metric} shows the comparative performance on various metrics for the 2-epoch training duration, where RKDO tends to show its strongest advantages.

\begin{table}[!ht]
\caption{Multi-Metric Comparison at 2 Epochs}
\label{tab:multi_metric}
\centering
\resizebox{\columnwidth}{!}{%
\begin{tabular}{lcccr}
\toprule
Dataset & Metric & RKDO & I-Con & Difference \\
\midrule
CIFAR-10 & Linear Eval & 0.2667 $\pm$ 0.0208 & 0.2723 $\pm$ 0.0139 & -2.08\% \\
CIFAR-10 & NMI & 0.1211 $\pm$ 0.0186 & 0.1255 $\pm$ 0.0179 & -3.48\% \\
CIFAR-10 & ARI & 0.0533 $\pm$ 0.0103 & 0.0550 $\pm$ 0.0102 & -3.10\% \\
CIFAR-10 & Neighborhood Acc & 0.1798 $\pm$ 0.0085 & 0.1806 $\pm$ 0.0125 & -0.44\% \\
\midrule
CIFAR-100 & Linear Eval & 0.0255 $\pm$ 0.0115 & 0.0170 $\pm$ 0.0035 & +50.00\% \\
CIFAR-100 & NMI & 0.5703 $\pm$ 0.0038 & 0.5678 $\pm$ 0.0075 & +0.43\% \\
CIFAR-100 & ARI & 0.0152 $\pm$ 0.0036 & 0.0161 $\pm$ 0.0027 & -5.15\% \\
CIFAR-100 & Neighborhood Acc & 0.0263 $\pm$ 0.0007 & 0.0278 $\pm$ 0.0023 & -5.51\% \\
\midrule
STL-10 & Linear Eval & 0.2225 $\pm$ 0.0350 & 0.1775 $\pm$ 0.0352 & +25.40\% \\
STL-10 & NMI & 0.1715 $\pm$ 0.0223 & 0.1749 $\pm$ 0.0164 & -1.97\% \\
STL-10 & ARI & 0.0702 $\pm$ 0.0140 & 0.0670 $\pm$ 0.0103 & +4.80\% \\
STL-10 & Neighborhood Acc & 0.1690 $\pm$ 0.0131 & 0.1744 $\pm$ 0.0210 & -3.13\% \\
\bottomrule
\end{tabular}%
}
\end{table}

\end{document}